\documentclass[sigconf]{acmart}

\usepackage{booktabs} 
\usepackage{graphicx}
\usepackage{ifthen}

\usepackage[utf8]{inputenc} 
\usepackage[T1]{fontenc}    
\usepackage{url}            
\usepackage{amsfonts}       
\usepackage{dsfont}

\usepackage{nicefrac}       
\usepackage{microtype}      
\usepackage{amsmath}
\usepackage{bm}
\usepackage{subcaption}
\usepackage{mathrsfs}
\usepackage{xcolor}
\usepackage{enumitem}
\usepackage{algorithm}
\usepackage{algorithmic}
\usepackage{multirow}
\usepackage{hyperref}
\usepackage{todonotes}

\usepackage{balance}

\newcommand {\out}[1]{}

\newcommand{\changed}[1]

\usepackage{mathtools}
\DeclarePairedDelimiterX{\norm}[1]{\lVert}{\rVert}{#1}

\newtheorem{theorem}{Theorem}

\copyrightyear{2021}

\begin{document}
\title{Scalable Personalised Item Ranking \\ through Parametric Density Estimation}

\author{Riku Togashi}
\affiliation{
  \institution{CyberAgent, Inc., Waseda University}
  \city{Tokyo} 
  \country{Japan}
}
\email{rtogashi@acm.org}

\author{Masahiro Kato}
\affiliation{
  \institution{CyberAgent, Inc.}
  \city{Tokyo} 
  \country{Japan}
}
\email{masahiro_kato@cyberagent.co.jp}

\author{Mayu Otani}
\affiliation{
  \institution{CyberAgent, Inc.}
  \city{Tokyo} 
  \country{Japan}
}
\email{otani_mayu@cyberagent.co.jp}

\author{Tetsuya Sakai}
\affiliation{
  \institution{Waseda University}
  \city{Tokyo} 
  \country{Japan}
}
\email{tetsuyasakai@acm.org}

\author{Shin'ichi Satoh}
\affiliation{
  \institution{CyberAgent, Inc.}
  \city{Tokyo} 
  \country{Japan}
}
\email{satoh@nii.ac.jp}

\renewcommand{\shortauthors}{
}


\begin{abstract} 
Learning from implicit feedback is challenging because of the difficult nature of the one-class problem: we can observe only positive examples.
Most conventional methods use a pairwise ranking approach and negative samplers to cope with the one-class problem.
However, such methods have two main drawbacks particularly in large-scale applications;
(1) the pairwise approach is severely inefficient due to the quadratic computational cost; and
(2) even recent model-based samplers (e.g. IRGAN) cannot achieve practical efficiency due to the training of an extra model.

In this paper, we propose a learning-to-rank approach, which achieves convergence speed comparable to the pointwise counterpart while performing similarly to the pairwise counterpart in terms of ranking effectiveness.
Our approach estimates the probability densities of positive items for each user within a rich class of distributions,
viz. \emph{exponential family}.
In our formulation,
we derive a loss function and the appropriate negative sampling distribution based on maximum likelihood estimation.
We also develop a practical technique for risk approximation and a regularisation scheme.
We then discuss that our single-model approach is equivalent to an IRGAN variant under a certain condition.
Through experiments on real-world datasets,
our approach outperforms the pointwise and pairwise counterparts in terms of effectiveness and efficiency.
\end{abstract}



\maketitle

\section{Introduction}
Recommender systems play a central role in alleviating information overload problems by helping users to discover content relevant to their interests.
To find relevant content from databases that are updated daily, implicit user feedback (e.g. clicks) is an important clue to users' recent interests because it is automatically collected and thus abundant.
However, the increase in the number of recommendation candidates leads to the sparsity of user feedback, overwhelming the accumulation of observed feedback.
Moreover, the use of implicit feedback incites the \emph{one-class problem}~\cite{pan2008one} in model training,
i.e., we can observe only positive examples.

For learning personalised ranking from positive-only data, a prevalent approach is to solve pairwise ranking problems~\cite{rendle2009bpr}.
The pairwise approach has the advantage in ranking effectiveness, but training is severely inefficient due to the quadratic computational cost of the expected loss (i.e. the risk) for each user.
By contrast, the pointwise approach can be trained efficiently in large-scale settings, but it returns less effective rankings.
Taking the advantage of the pointwise approach, non-sampling
methods~\cite{chen2019efficient,chen2020efficient} based on mini-batch stochastic gradient descent (SGD) have
demonstrated excellent convergence time with competitive effectiveness.
In addition, the SGD-based approaches enable the use of non-uniform weights for negative instances, which is often infeasible when using alternating least squares (ALS)~\cite{pan2008one,pan2009mind}.

Missing pairs of users and items could complement the positive-only data because these may provide users' negative signals.
When dealing with large-scale data, it is challenging to find informative negatives from a large number of candidates.
Therefore, a negative sampling/weighting strategy is a critical factor for efficient model training~\cite{rendle2009bpr,zhang2013optimizing,rendle2014improving}.
For alleviating the training inefficiency of the pairwise approach, 
previous studies explored strategies based on item popularity~\cite{pan2008one}, predicted scores of the ranker model~\cite{zhang2013optimizing,rendle2014improving}, an adversarial sampler model~\cite{IRGAN}, and reinforcement learning~\cite{ding2019rein}. 
Since a sampling strategy aims to find items with high preference scores from unobserved samples,
the strategy will behave like a recommender model.
Hence, advanced strategies often require an extra sampler model which learns to rank negative candidates according to their informativeness.
The IRGAN framework~\cite{IRGAN}, which is based on generative adversarial networks (GANs)~\cite{GAN}, provides an optimised negative sampler (i.e. the generator) for the ranker (i.e. the discriminator).
However, such a model-based sampler faces two challenges that result from the training of the two models.
First, the training process introduces hyper-parameters that are difficult to tune, such as the number of iterations for training the discriminator and generator in each epoch.
Second, it leads to inefficient training because multiple models need to be trained.

Herein, we propose a learning-to-rank approach for personalised ranking from implicit feedback based on parametric density estimation.
We derive a loss function as a likelihood of generative distributions of user feedback.
In our formulation, we can obtain the negative sampling distribution for a user as the estimated generative distribution of the user's feedback.
To enable the use of the various models such as graph neural networks (GNNs), 
we also propose risk approximation and regularisation techniques that encourage stable and efficient model training using SGD-based empirical risk minimisation.
Furthermore, we show the equivalence between our approach and an IRGAN variant based on Wasserstein GAN (WGAN)~\cite{arjovsky2017wasserstein} under a certain condition.
Based on this, we propose a variant of our approach, which estimates the distribution of positive items over \emph{unobserved items} through the minimisation of the Wasserstein distance.
We then show that our approach behaves theoretically similarly to the pairwise ranking approach with a model-based negative sampler while enabling drastically faster convergence in empirical experiments.

\section{Background and Related Work}
\subsection{Problem Formulation}
For a set of users $\mathcal{U}$ and a set of items $\mathcal{I}$,
we observe $N$ i.i.d. user feedback logs $(u_1, i_1, y_1),\dots,(u_N, i_N, y_N)$, where $u \in \mathcal{U}$, $i \in \mathcal{I}$, and $y \in \{-1, +1\}$ are user, item, and class labels, respectively.
However, if only implicit feedback is available, we can observe only positive samples, in other words, those labelled as $y=+1$.
Given a user, a ranker sorts items in the order in which they are most likely to interact with the user.

\subsection{Paradigms of Personalised Ranking}
In top-$K$ personalised ranking, most pointwise methods solve a binary classification problem to predict if a user prefers an item ~\cite{pan2008one,hu2008collaborative,koren2009matrix}; we will refer to this as the \emph{class-posterior probability estimation (CPE)} approach.
The CPE approach, which estimates class-posterior $p(y=+1|u,i)$, often impairs ranking performance,
whereas it is useful for the tasks such as CTR prediction.
On the other hand, owing to pointwise loss functions, the CPE approach has an advantage in the training efficiency compared with the pairwise counterpart.
Methods based on the CPE approach often introduce weighting strategies for selecting important unobserved samples.
Early works explored non-uniform weighting schemes such as item popularity~\cite{pan2008one,pan2009mind}, but non-uniform weights lead to severe computational costs.
For the sake of training efficiency, static and uniform weights have been adopted widely~\cite{hu2008collaborative}. 
Recent studies have proposed efficient methods that use non-uniform weights based on ALS~\cite{he2019fast} and SGD~\cite{chen2019efficient,chen2020efficient,chen2020jointly}.
Several studies have explored weighting strategies for the CPE approach based on item popularity, user exposure models~\cite{liang2016modeling}, and social information~\cite{chen2019samwalker}.
As the inefficiency issue of adaptive weighting~\cite{liang2016modeling} is problematic in a large-scale setting,
Chen et al.~\cite{chen2020fast} have proposed an efficient optimisation method with adaptive weighting strategy for matrix factorisation (MF).

In contrast to the CPE approach that estimates $p(y=+1|u,i)=p(i|u,y=+1)p(y=+1|u)/p(i|u)$,
the recently-proposed \emph{density ratio estimation (DRE)} approach~\cite{togashi2021density} directly estimates the density ratio $p(i|u,y=+1)/p(i|u)$
through the minimisation of the Bregman divergences~\cite{sugiyama2012density}, which can be empirically approximated with positive-unlabelled data. 
As the main distinction of the DRE approach from its CPE counterpart,
the ranker model in the DRE approach does not need to estimate class prior $p(y=+1|u)$ (i.e. user's activity),
which is a constant value and does not affect the predicted ranking for $u$;
thus, it avoids the estimation error of $p(y=+1|u)$ in the ranker model.
Several studies~\cite{menon2016linking,togashi2021density} have shown that the advantage of the DRE-based loss functions for top-$K$ ranking problems by weighing the items in the top of ranked lists in contrast to the CPE approach.
Togashi et al.~\cite{togashi2021density} empirically demonstrated the advantage of the DRE-based approach in terms of efficiency and effectiveness for top-$K$ personalised recommendation.
Although their framework can be utilised with an arbitrary weighting strategy, their proposed adaptive and ranking-aware weighting strategies are heuristically designed.

By directly modelling the generative process of implicit feedback from each user, estimating $p(i|u,y=+1)$ is a natural direction;
we call this approach the \emph{density estimation (DE)} approach.
The intuition behind estimating $p(i|u,y=+1)$ is that,
as $\int p(i|u,y=+1) di=1$ holds for each user, the items must compete for the limited budget of the total probability mass.
This one-sum property is suitable for top-$K$ ranking tasks because it induces relativity in the model predictions~\cite{yang2011collaborative,liang2018variational}.
The multinomial logit model (a.k.a. softmax model)~\cite{mcfadden1973conditional,manski1975maximum,blanc2018adaptive} is widely adopted as the model of $p(i|u,y=+1)$:
\begin{align}
   q(i|u) = \frac{\exp(f_u(i))}{\sum_{i \in \mathcal{I}}\exp(f_u(i))}
\end{align}
where $f_u(i) \in \mathbb{R}$ is $u$'s preference score for $i$.
Various methods have been proposed based on this parametric density function,
such as collaborative competitive filtering (CCF)~\cite{yang2011collaborative} and VAE-based collaborative filtering~\cite{liang2018variational,liu2019deep,shenbin2020recvae}.
For quantifying the distance between the true and estimated distributions, Kullback-Leibler divergence (KLD)~\cite{kullback1951information} is often utilised.
Tanielian et al.~\cite{tanielian2019relaxed} recently proposed a DE approach based on the maximum likelihood estimation (MLE) of softmax density functions.
However, their MLE formulation leads to intractable log-partition $\log \sum_{i \in \mathcal{I}}\exp(f_u(i))$ in the log-likelihood function, which makes SGD-based optimisation difficult for large-scale settings.
They also discussed that their approach can be regarded as the minimisation of KLD. 
On the other hand, Tolstikhin et al.~\cite{tolstikhin2017wasserstein} have shown that, when using a VAE-based method, minimising the Wasserstein distance (WD)~\cite{villani2008optimal} is equivalent to optimising the evidence lower bound (ELBO).
This fact allows us to use the WD rather than the KLD, which is inappropriate for the high-dimensional and sparse data that we have when dealing with user implicit feedback.
Based on this fact, recent studies have examined approaches that estimate the probability densities by minimising the WD~\cite{liu2019deep,yu2019wassrank,zhang2020wasserstein,meng2020wasserstein}.
However, VAE-based methods restrict the use of flexible models, such as MF~\cite{pan2008one,hu2008collaborative,koren2009matrix,rendle2009bpr} and GNNs~\cite{NGCF,he2020lightgcn}.
Furthermore, the VAE-based models often include a hidden layer that is as large as the number of items, making these models computationally costly for a large-scale setting.

The pairwise approach introduces another binary classification problem of predicting which of any given two items a user will prefer~\cite{burges2005learning,rendle2009bpr,agarwal2011infinite,boyd2012accuracy,christakopoulou2015collaborative}.
As the pairwise approach avoids predicting the absolute relevance of a user-item pair,
it is often more appropriate than the CPE approach for top-$K$ recommendation tasks.
While this approach is advantageous in terms of effectiveness,
it has a serious disadvantage in the convergence time.
Empirical risk minimisation based on U-statistics (e.g. Kendall's $\tau$) is practically infeasible in large-scale settings due to its quadratic computational cost with respect to $|\mathcal{I}|$.
Some techniques to alleviate this problem are available~\cite{papa2015sgd,clemenccon2016scaling},
but such techniques are not well examined in IR and recommendation literature.
Previous studies have proposed the top-emphasised pairwise approaches based on order statistics~\cite{weston2012latent,weston2013learning}, push-based losses~\cite{christakopoulou2015collaborative} and importance sampling~\cite{lian2020personalized}.

Our approach returns to the simple MLE of $p(i|u,y=+1)$ within an exponential family~\cite{canu2006kernel,sriperumbudur2017density,arbel2018kernel,dai2019exponential,dai2019kernel},
which includes the softmax model adopted in the conventional DE approaches.
Using the penalised MLE of parametric density functions,
we derive a loss function and an appropriate sampling/weighting strategy from a perspective of empirical risk minimisation.
Our methodology is applicable to most common model architectures like MF and GNNs.
Moreover, it offers two advantages in terms of training efficiency when compared with the pairwise approach;
while working in a pointwise fashion, it provides the optimal negative sampler without an extra model.

\subsection{Negative Sampling and Generative Adversarial Networks\label{section:related-gan}}
Because the choice of negative sampling strategy has a great impact on both efficiency and effectiveness,
various approaches have been explored based on item popularity, ranker's predicted scores~\cite{zhang2013optimizing,rendle2014improving}, additional view data~\cite{ding2019sampler}, optimised sampler models~\cite{IRGAN}, and reinforcement learning~\cite{ding2019rein}.

Early works~\cite{zhang2013optimizing,rendle2014improving} have developed ranker-based sampling strategies, which select the negatives scored higher by the ranker for training.
Ding et al.~\cite{ding2020simplify} have recently explored a simple and robust negative sampling strategy, which provides an efficient and effective implementation of negative sampling and simplifies model-based samplers for practical applications.
In contrast to a ranker-based sampler, Wang et al.~\cite{IRGAN} proposed a learning-to-rank framework based on GANs and
formulated ranking tasks as a min-max game between two rankers, a discriminator and a generator.
For personalised recommendation tasks, the generator model is mainly regarded as the negative sampler for the ranker.
Hence, GAN-based methods have often been explored in the context of the negative sampling strategies~\cite{zhang2013optimizing,fan2019deep,wang2019adversarial,NMRN,liu2020cofigan,guo2020ipgan}.
On the other hand, since GANs originally aim to estimate the generative distribution of observed data,
when using the generator as the ranker model,
GAN-based methods are more like the DE approach; the softmax model is widely adopted for the generator.
The disadvantage of the GAN-based approach is the computational cost of the following steps in the model training: 
(1) Training the two models for the ranker and sampler.
(2) Negative sampling from the generator's distribution over all items.
These steps also pose a challenge for the practical implementation of large-scale recommender systems. 
Moreover, several prior studies~\cite{ding2019rein,wang2020reinforced} have reported that GAN-based negative samplers can underperform ranker-based samplers such as dynamic negative sampling (DNS)~\cite{zhang2013optimizing}, which adaptively selects the negative item scored highest by the current ranker.

In this study, we aim to develop a risk function incorporating the concept of GAN-based samplers.
Remarkably, under a certain assumption on the ranker (discriminator), our MLE-based approach is equivalent to a variant of IRGAN with a regularised generator while omitting the generator model in the formulation of GANs.
Whereas our motivation to simplify the recent negative samplers is similar to that of Ding et al.~\cite{ding2020simplify}, we shall propose a risk function rather than the algorithm for negative sampling.

\section{Proposed Method\label{section:mle-de}}
\subsection{Parametric Density Functions\label{kefe}}
We formulate the personalised ranking as the estimation of the densities of items conditioned on a user and the positive class, namely, $p(\cdot|u,y=+1)$ on $\mathcal{I}$.
To this end, we first specify the model of $p(\cdot|u,y=+1)$ with an exponential family parametrised by $f_u$.
\begin{align}
  \label{eq:kef}
  p_f(i|u) \coloneqq p_0(i)\exp(f_u(i) - A(f_u)),
\end{align}
where $f_u$ denotes the sufficient statistics, and $p_0$ is a base density function on $\mathcal{I}$. 
Here, $A(f_u) \coloneqq \log \int p_0(i)\exp(f_u(i))di$ is the normalisation term that ensures $\int_i p_f(i|u) di = 1$ (a.k.a. log-partition function).
In the context of personalised ranking, $f_u(i)$ is the scoring function to learn, which represents $u$'s preference score for $i$.
In typical models, such as MF and GNNs,
each score is computed as the similarity between $u$ and $i$ in the latent feature space.
\begin{align}
  f_u(i) \coloneqq {\bf e}_u^{\top}{\bf e}_i,
\end{align}
where ${\bf e}_u \in \mathbb{R}^d$ and ${\bf e}_i \in \mathbb{R}^d$ are the $d$-dimensional latent embeddings of $u$ and $i$, respectively.
Considering practical applications, we assume that $p_0(i)$ is a uniform probability for all $i$;
as a result, $p_f(i|u) \propto \exp(f_u(i))$ holds, and therefore, the order of $p_f(i|u)$ is determined only by the order of $f_u(i)$.
This allows us to utilise approximate nearest neighbour (ANN) techniques~\cite{liu2005investigation} to
efficiently perform the top-$K$ maximum inner product search (MIPS) for a user.
This is essential for real applications in which the number of items is large~\cite{rendle2020neural}.   

\subsection{Penalised Maximum Likelihood Estimation\label{penalised-mle}}
To learn $f_u(i)$ for each $u$ from the observed implicit feedback,
we formulate the MLE of density functions of positive items for users.
Because the na\"{i}ve MLE is ill-posed and can impair generalisation performance,
we introduce a penalisation term based on the KLD between the estimated distribution and a uniform distribution.
We can express the empirical risk based on the negative log-likelihood function with KLD penalisation as follows:
\begin{align}
  R(f)
  &\coloneqq \mathbb{E}_{u}\left[\widehat{\mathbb{E}}_{+} \left[-\log p_f(i|u)\right] + KLD(p_f(\cdot|u)||p_0(\cdot)) \right] \\ \label{eq:plugin-pf-def}
  &= \mathbb{E}_{u}\left[\widehat{\mathbb{E}}_{+} \left[-\log p_0(i) - f_u(i) + A(f_u)\right] + KLD(p_f(\cdot|u)||p_0(\cdot)) \right] \\\label{eq:naive-likelihood}
  &= \mathbb{E}_{u}\left[\widehat{\mathbb{E}}_{+} \left[- f_u(i)\right] + A(f_u) + KLD(p_f(\cdot|u)||p_0(\cdot)) \right] + C,
\end{align}
where $\widehat{\mathbb{E}}_{+}=\widehat{\mathbb{E}}_{p(i|u,y=+1)}$ is the empirical expectation over the training positive items for a given user $u$,
and $C$ is a constant value with respect to the training of $f_u$.
In Eq.~(\ref{eq:plugin-pf-def}), we used the definition of $p_f(i|u)$ in Eq.~(\ref{eq:kef}).
The problem here is that $A(f_u)=\log \int p_0(i)\exp(f_u(i))di$ is intractable, and thus, $R(f)$ is difficult to be approximated through mini-batch SGD.
To solve this, we transform the KLD between $p_f(\cdot|u)$ and $p_0(\cdot)$ as follows:
\begin{align}
  KLD(p_f(\cdot|u)||p_0(\cdot)) &\coloneqq \int p_f(i|u)\log(p_f(i|u)/p_0(i)) di \\
  &= \int p_f(i|u)\log \exp(f_u(i)-A(f_u)) di \\
  &= \int p_f(i|u)f_u(i) di - A(f_u) \\
  &= \mathbb{E}_{p_f(i|u)}\left[f_u(i)\right] - A(f_u).
\end{align}
Plugging this into Eq.~(\ref{eq:naive-likelihood}), we can reformulate the risk as follows:
\begin{align}
  R(f) &= \mathbb{E}_{u}\Bigg[\widehat{\mathbb{E}}_{+} \left[ -f_u(i)\right] + A(f_u)
  + \mathbb{E}_{p_f(i|u)}\left[f_u(i)\right] - A(f_u) \Bigg] \\ \label{eq:nll-risk}
  &= \mathbb{E}_{u}\left[\widehat{\mathbb{E}}_{+} \left[ -f_u(i)\right] + \mathbb{E}_{p_f(i|u)}\left[f_u(i)\right]  \right].
\end{align}
Remarkably, the appropriate distribution for negative sampling under ranker $f_u(\cdot)$ is exactly $p_f(i|u)$ in our formulation.
Note that we use all samples including (even observed) positive ones to compute the second term of the RHS in Eq.~(\ref{eq:nll-risk}), although it corresponds to the risk for negative instances in the conventional pointwise or pairwise approaches;
we shall discuss this point in Section~\ref{section:disjoint-support}.

\subsection{Risk Approximation\label{section:risk-approx}}
For sampling negatives from ranker-dependent $p_f(i|u)$, we need to compute each user's preference scores for all items;
this is infeasible in a large-scale setting.
Considering this, our approach follows the two-phased sampling strategy~\cite{zhang2013optimizing,rendle2014improving}.
We first sample a set of item candidates from a static sampling distribution and then evaluate the preference scores for them on the fly.
In addition, we shall tackle another challenge in negative sampling with ranker-dependent distributions.
That is, because $p_f(i|u)$ changes with the update to $f$ in each training step,
costly Monte-Carlo sampling is often required at every training step to accurately approximate the gradient of $\mathbb{E}_{p_f(i|u)}[f_u(i)]$.
To achieve better risk approximation without sampling from $p_f(i|u)$,
we transform $R(f)$ by leveraging explicit parametric model $p_f(i|u)$ defined in Eq.~(\ref{eq:kef}).
By introducing importance sampling, we can obtain the following risk:
\begin{align}
  R(f)  &= \mathbb{E}_{u}\left[\widehat{\mathbb{E}}_{+} \left[ -f_u(i)\right] + \mathbb{E}_{p_0}\left[\frac{p_f(i|u)}{p_0(i)}f_u(i)\right]  \right] \\
  &= \mathbb{E}_{u}\left[\widehat{\mathbb{E}}_{+} \left[ -f_u(i)\right] + \frac{\mathbb{E}_{p_0}\left[\exp(f_u(i))f_u(i)\right]}{\exp{A(f_u)}}\right]  \\
  &= \mathbb{E}_{u}\left[\widehat{\mathbb{E}}_{+} \left[ -f_u(i)\right] + \frac{\mathbb{E}_{p_0}\left[\exp(f_u(i))f_u(i)\right]}{\mathbb{E}_{p_0}\left[\exp(f_u(i))\right]}\right]. 
\end{align}
In the third equation, we use $\exp A(f_u)=\int p_0(i)\exp(f_u(i))di=\mathbb{E}_{p_0}[\exp(f_u(i))]$.
Finally, we reach the empirically approximated risk with a given mini-batch.
\begin{align}
  \notag
  &\widehat{R}(f, \mathcal{U}_B, \mathcal{I}_B) \\ \label{eq:risk-estimator}
  &\coloneqq \frac{1}{|\mathcal{U}_{B}|}\sum_{u \in \mathcal{U}_B}\left(-\frac{1}{|\mathcal{I}_{u}^{+}|}\sum_{i \in \mathcal{I}_{u}^{+}} f_u(i) + \frac{\sum_{i \in \mathcal{I}_B}\exp(f_u(i))f_u(i)}{\sum_{i \in \mathcal{I}_B}\exp(f_u(i))} \right),
\end{align}
where $\mathcal{U}_B$ and $\mathcal{I}_B$ are the sets of sampled users and items, respectively.
Here, $\mathcal{I}_{u}^{+}$ is the set of observed positive items for $u$.
Note that $\widehat{R}(f, \mathcal{U}_B, \mathcal{I}_B)$ is biased for general sampling distributions of $\mathcal{I}_B$.
However, we found in our preliminary experiments that the corrected risk estimator with inverse sampling probabilities did not perform well due to a large variance of the estimator, and we therefore omit it here for the sake of simplicity.

For mini-batch sampling, we adopted the same user-based sampling strategy as is used in conventional approaches~\cite{liang2018variational,togashi2021density}.
We uniformly sample $|\mathcal{U}_B|$ users and then sample all positive items $\mathcal{I}_{u}^{+}$ for each user.
For computing $\widehat{R}(f, \mathcal{U}_B, \mathcal{I}_B)$, we build a set of items $\mathcal{I}_B$ in a mini-batch as unique sampled items. 
We leave the use of advanced samplers for $\mathcal{I}_B$ as future work.

In practice, we minimise $\widehat{R}(f, \mathcal{U}_B, \mathcal{I}_B)$ with the $L_2$ regularisation for the latent vectors as follows: 
\begin{align}
  \label{eq:pde-obj}
  \min_f \widehat{R}(f, \mathcal{U}_B, \mathcal{I}_B) + \lambda \mathcal{R}(f, \mathcal{U}_B, \mathcal{I}_B),
\end{align}
where $\mathcal{R}(f, \mathcal{U}_B, \mathcal{I}_B)$ is the $L_2$ regularisation term, and $\lambda$ is the hyper-parameter that controls the intensity of the regularisation.

\subsection{Norm Clipping Technique\label{section:norm-clip}}
Intuitively, owing to the one-sum property of $p_f(i|u)$,
the spike of the probability mass for a certain item dominates the total probability mass.
This effect leads to overfitting or a large variance in $\widehat{R}(f, \mathcal{U}_B, \mathcal{I}_B)$, and thus destabilises the model training.
This observation suggests the following smoothness condition for $p_f(i|u)$: 
\begin{align}
  \label{eq:smoothness-condition}
  \max_{i,i' \in \mathcal{I}} |p_f(i|u)-p_f(i'|u)| \leq L, 
\end{align}
where $L \in (0, \infty)$ is the upper bound of the gap between the probability mass assigned to two items.
Owing to the monotonicity of the exponential function, we can rewrite this condition as follows:
\begin{align}
  \label{eq:lipschitz-condition-critic}
  \max_{i,i' \in \mathcal{I}} |p_f(i|u)-p_f(i'|u)| \leq L &\Leftrightarrow \max_{i,i' \in \mathcal{I}} |f_u(i)-f_u(i)| \leq \tilde{L} \\  
  &\Leftrightarrow \max_{i,i' \in \mathcal{I}} |{\bf e}_{u}^{\top}({\bf e}_{i}-{\bf e}_{i'})| \leq \tilde{L} \\
  &\Leftarrow \max_{v \in \{u\} \cup \mathcal{I}} \norm{{\bf e}_{v}} \leq \bar{n},
\end{align}
where $\tilde{L}$ is the upper bound of the gap between two preference scores, and 
$\bar{n} \in (0, \infty)$ is the upper bound of the $L_2$ norms of the latent vectors.
To satisfy the smoothness condition,
we ensure that the latent vectors for all users and items have upper-bounded $L_2$ norms.
It is worth noting that the $L_2$ regularisation is not appropriate for ensuring the condition
because it penalises the average of the $L_2$ norms of latent vectors rather than their maxima.
Therefore, we propose an $L_2$ norm clipping technique.
For methods based on MF, we define the following clipping operator:
\begin{align}
  \text{clip}({\bf e}, \bar{n}) \coloneqq \begin{cases}
    {\bf e}, \,\,\,\,\,\, &\text{if } \norm{{\bf e}} \leq \bar{n} \\
    \frac{\bar{n}}{\norm{{\bf e}}}{\bf e}, \,\,\,\,\,\, &\text{if } \norm{{\bf e}} > \bar{n}
  \end{cases},
\end{align}
where ${\bf e} \in \mathbb{R}^d$ denotes a latent vector.
We apply the clipping operator to each latent vector after updating it in each training step.

This technique is also applicable to the recently-proposed GNN-based method, LightGCN~\cite{he2020lightgcn}, which considers high-order smoothness on the user-item graph.
In LightGCN, each final latent vector for prediction is computed as the linear combination of the trainable latent vectors in the first layer:
\begin{align}
  {\bf e}_u = \sum_{v \in \mathcal{G}} \beta_{v,u}{\bf e}_v^{(1)},\,\,\, {\bf e}_i = \sum_{v \in \mathcal{G}} \beta_{v,i}{\bf e}_v^{(1)}, 
\end{align}
where $\mathcal{G}=\mathcal{U}\cup\mathcal{I}$ is the set of nodes in the user-item graph,
and ${\bf e}_v^{(1)}$ denotes the latent vector for $v \in \mathcal{G}$ in the first layer.
Here, $\beta_{v,u}$ and $\beta_{v,i}$ are non-negative coefficients in the linear combination.
Therefore, when clipping the latent vectors in the first layer, the norms of the final vectors are also upper-bounded.

\section{Theoretical Discussion}
In this section, we first discuss the relationship between our approach and IRGAN;
based on the discussion, we propose a variant of our DE approach based on the WD. 
We then discuss its connection to the pairwise approach.

\subsection{Relationship to IRGAN\label{section:rel-wgan}}
Previous studies pointed out the connection between GANs and density estimation~\cite{arbel2018kernel,dai2019exponential,dai2019kernel}.
Following these studies, we discuss the similarity between our single-model approach and an IRGAN variant based on the WGAN~\cite{arjovsky2017wasserstein}.

We first assume a method based on WGAN, namely, \emph{IRWGAN}.
We can express the dual form of the Wasserstein-1 distance between the true and estimated distributions $P_u$ and $Q_u$ for a user as follows:
\begin{align}
  \label{eq:wass-dist-def}
  WD(P_u||Q_u) \coloneqq \sup_{\norm{f_u}_{L\leq 1}} \mathbb{E}_{i \sim P_u}\left[f_u(i)\right] - \mathbb{E}_{i \sim Q_u}\left[f_u(i)\right],
\end{align}
where $\norm{f_u}_{L\leq 1}$ requires the 1-Lipschitz continuity of $f_u$, namely, $\max_{i,i'\in \mathcal{I}}|f_u(i)-f_u(i')| \leq 1$; we here assume that $\mathcal{I}$ is a discrete metric space.

By formulating density estimation as the WD minimisation under KLD regularisation for the generator,
we can obtain the min-max objective of IRWGAN for user $u$ as follows:
\begin{align}
\notag
  &\min_{q_u \in \mathcal{P}} KLD(q_u||p_0) + WD(P_u||Q_u) \\\label{eq:wass-obj}
  = &\min_{q_u \in \mathcal{P}} \left\{KLD(q_u||p_0) + \sup_{f_u \in \mathcal{F}_1}\left\{\widehat{\mathbb{E}}_{+}\left[f_u(i)\right] - \mathbb{E}_{q_u(i)}\left[f_u(i)\right]\right\}\right\},
\end{align}
where $q_u \in \mathcal{P}$ and $f_u \in \mathcal{F}_1$ are the generator and discriminator for $u$, respectively.
We assume that $\mathcal{P}$ is the set of probability measures supported on $\mathcal{I}$,
and $\mathcal{F}_1$ is a convex subset of 1-Lipschitz functions.

We here show that IRWGAN with convex $\mathcal{F}_1$ leads to the same ranker as the proposed approach in Section~\ref{section:mle-de}.
To this end, by following Theorem 1 of Farnia et al.~\cite{farnia2018convex}, we first consider the duality of the IRWGAN objective.
\begin{theorem}[Strong Duality of IRWGAN]
\label{theorem:duality-irwgan}
Suppose a set $\mathcal{F}_1$ is convex. Then, the strong duality of the IRWGAN holds, namely,
\begin{align}
\notag
   &\text{RHS of Eq.~(\ref{eq:wass-obj})} \\\label{eq:max-min-irwgan}
   = &\sup_{f_u \in \mathcal{F}_1} \left\{\widehat{\mathbb{E}}_{+}\left[f_u(i)\right] + \min_{q_u \in \mathcal{P}}\left\{KLD(q_u||p_0) - \mathbb{E}_{q_u(i)}\left[f_u(i)\right]\right\}\right\}.
\end{align}
\end{theorem}
\begin{proof}
We show the applicability of von Neumann-Fan minimax theorem (Thereom 2 of Borwein's work~\cite{borwein2016very}), which guarantees the strong duality of min-max objectives under several conditions.
The space $\mathcal{P}$ of probability measures on compact $\mathcal{I}$ is convex and weakly compact~\cite{billingsley2013convergence}.
The set $\mathcal{F}_1$ of 1-Lipschitz continuous functions is convex as we assumed.
The objective in Eq.~(\ref{eq:wass-obj}) is linear (i.e. concave) in $f_u$ while
it is lower semi-continuous and convex in $q_u$ owing to the KLD regularisation for $q_u$.
Therefore, the strong duality of IRWGAN holds by applying the minimax theorem.
\end{proof}

By using Theorem~\ref{theorem:duality-irwgan}, we obtain the risk for $f$ with the optimal generators $q_u^*$ based on the max-min objective (i.e. the RHS of Eq.~(\ref{eq:max-min-irwgan})) as follows:
\begin{align}
  \label{eq:wgan-risk}
  R_{\text{wgan}}^{*}(f) &\coloneqq \mathbb{E}_{u}\left[\widehat{\mathbb{E}}_{+}\left[-f_u(i)\right] + \mathbb{E}_{q_u^*(i)}\left[f_u(i)\right]\right].
\end{align}
Here, based on Eq.~(\ref{eq:max-min-irwgan}), the optimal distribution $q_u^{*}(i)$ for $u$ is the one that maximises objective $L_u(q_u) \coloneqq \mathbb{E}_{q_u(i)}\left[f_u(i)\right] - KLD(q_u||p_0)$.
To obtain $q_u^{*}(i)$, we solve the following maximisation problem with respect to $q_u$:
\begin{align}
  \notag
  &\max_{q_u \in \mathcal{P}} \mathbb{E}_{q_u}\left[f_u(i)\right] - KLD(q_u||p_0) \\
  &= \max_{q_u \in \mathcal{P}} \int q_u(i)f_u(i) di - \int q_u(i)\left(\log q_u(i) - \log p_0\right) di \\
  &= \max_{q_u \in \mathcal{P}} \int q_u(i)\left(f_u(i) + \log p_0 - \log q_u(i)\right) di. 
\end{align}
By considering the differentiation with respect to density $q_u(i)$,
$\nabla_{q_u(i)} L_u(q_u) = 0$ holds for each $i$ at the optimal $q_u^{*}$ owing to the strong concavity of $L_u(\cdot)$.
Thus, we can obtain the optimal density $q_u^{*}(i)$ as the following closed form. 
\begin{align}
  \nabla_{q_u(i)} L_u(q_u^*) = 0 &\Leftrightarrow f_u(i) + \log p_0 - \log q_u^*(i) - 1 = 0 \\
  & \Leftrightarrow \log q_u^*(i) = f_u(i) + \log p_0 - 1 \\
  & \Leftrightarrow q_u^*(i) \propto p_0 \exp (f_u(i)).
\end{align}
Therefore, $q_u^{*}(i)$ is exactly equivalent to $p_f(i|u)$ defined in Eq.~(\ref{eq:kef}).
By substituting $q_u^{*}(i)=p_f(i|u)$ into Eq.~(\ref{eq:wgan-risk}),
we immediately obtain the same risk as $R(f)$ in Eq.~(\ref{eq:nll-risk}).

Based on this result, if we aim to obtain discriminator $f$ in IRWGAN with convex $\mathcal{F}_1$ as the ranker,
we can omit the generator model by obtaining the optimal sampling distributions based on $f_u$ as a closed form.
It allows us to keep the sampler up-to-date at every training step while reducing computational costs and hyper-parameters.
We can utilise the norm clipping technique to obtain the $\tilde{L}$-Lipschitz continuity of $f_u$ (the RHS of Eq.~(\ref{eq:lipschitz-condition-critic}).
In other words, with the norm clipping technique,
our approach may minimise $\tilde{L} \cdot WD(P_u||Q_u)$ with KLD regularisation.
It should be noted that the feasible set of $f_u$ (i.e. $\mathcal{F}_1$) is generally \emph{not} convex when using complex model architectures (probably including LightGCN),
and it may lead to a duality gap in yielding the max-min form in Theorem~\ref{theorem:duality-irwgan} to some extent.

\subsection{Estimating Densities on Unobserved Items\label{section:disjoint-support}}
In the ranking evaluation, we do not use the training items (i.e. observed positive items) to create a ranked list for a user.
Considering this, we propose a variant of our approach which estimates $p(\cdot|u,y=+1)$ on $\mathcal{I}\setminus\mathcal{I}_u^{+}$ rather than $p(\cdot|u,y=+1)$ on $\mathcal{I}$.
The problem here is that the empirical distribution of positive items has a support $\mathcal{I}_u^{+}$ which is disjoint with that of the distribution to be estimated; namely, $\mathcal{I}\setminus\mathcal{I}_u^{+}$.
However, based on the discussion in Section~\ref{section:rel-wgan}, our approach may minimise the WD, which can be defined even between two distributions with disjoint supports~\cite{arjovsky2017wasserstein}. 
Therefore, our approach can safely estimate the density function by minimising the WD between two disjoint distributions over observed and unobserved items, whereas many divergences (e.g. KLD) are inappropriate for such distributions.

Assuming that our approach minimises the WD,
we can obtain the risk for the density estimation on unobserved items by defining $p_f(\cdot|u)$ and $p_0(\cdot)$ on $\mathcal{I}\setminus\mathcal{I}_u^{+}$.
Thus, we can also obtain the risk estimator by simply replacing $\mathcal{I}_B$ with $\mathcal{I}_B\setminus\mathcal{I}_u^{+}$ in the RHS of Eq.~(\ref{eq:risk-estimator}) as follows:
\begin{align}
  \notag
  &\widehat{R}_{\text{WD}}(f, \mathcal{U}_B, \mathcal{I}_B) \\ \label{eq:risk-estimator-ds}
  &\coloneqq \frac{1}{|\mathcal{U}_{B}|}\sum_{u \in \mathcal{U}_B}\left(-\frac{1}{|\mathcal{I}_{u}^{+}|}\sum_{i \in \mathcal{I}_{u}^{+}} f_u(i) + \frac{\sum_{i \in \mathcal{I}_B\setminus\mathcal{I}_{u}^{+}}\exp(f_u(i))f_u(i)}{\sum_{i \in \mathcal{I}_B\setminus\mathcal{I}_{u}^{+}}\exp(f_u(i))} \right).
\end{align}

\subsection{Relationship to Pairwise Approach\label{section:rel-pairwise}}
We shall discuss how our proposed risk relates to that of pairwise approaches.
The user-conditional empirical risk of a pairwise approach with negative item sampler $Q_u$ can be defined as follows:
\begin{align}
 R_{\text{pair}}(f_u) &\coloneqq \widehat{\mathbb{E}}_{i \sim P_{u}}\mathbb{E}_{i' \sim Q_u}\left[s(f_u(i')-f_u(i))\right]
\end{align}
where $s(\cdot)=\log(1 + \exp(\cdot))$ is the softplus function;
here, $s(f_u(i')-f_u(i))$ is a widely adopted differentiable surrogate of the 0-1 loss, $\mathds{1}\{f_u(i') > f_u(i)\}$~\cite{rendle2009bpr,christakopoulou2015collaborative}.
Using the convexity of the softplus function, we can evaluate the following lower bound of the pairwise risk by applying Jensen's inequality.
\begin{align}
\label{eq:bpr-ub-proposed}
  R_{\text{pair}}(f_u) &\geq s\left(\widehat{\mathbb{E}}_{i \sim P_u}\left[-f_u(i)\right] + \mathbb{E}_{i \sim Q_u}\left[f_u(i)\right]\right).
\end{align}
Therefore, when using $p_f(i|u)$ on $\mathcal{I}\setminus\mathcal{I}_u^{+}$ for $Q_u$,
the RHS in Eq.~(\ref{eq:bpr-ub-proposed}), the lower bound of the pairwise risk corresponds to the proposed WD-based risk for a given user.
Because $p_f(i|u)$ maximises the RHS of Eq.~(\ref{eq:bpr-ub-proposed}) under the KLD regularisation as shown in Section~\ref{section:rel-wgan},
our approach minimises the upper bound of the lower bound of the pairwise risk.
Here, we let $\mu$ be the short-hand for $\widehat{\mathbb{E}}_{i \sim P_{u}}\mathbb{E}_{i' \sim Q_u}[f_u(i')-f_u(i)]$.
By assuming the $\tilde{L}$-Lipschitz continuity of $f_u$,
we can obtain the upper bound of the gap between the RHS and LHS of Eq.~(\ref{eq:bpr-ub-proposed}), i.e.,
the worst-case deviation between our approach and a pairwise counterpart, as follows:
\begin{align}
  |R_{\text{pair}}(f_u) - s(\mu)| &\leq \widehat{\mathbb{E}}_{i \sim P_{u}}\mathbb{E}_{i' \sim Q_u}[|s(f_u(i')-f_u(i)) - s(\mu)|] \\ \label{eq:jgap-lipschitz-softplus} 
  &\leq \widehat{\mathbb{E}}_{i \sim P_{u}}\mathbb{E}_{i' \sim Q_u}[|f_u(i') - f_u(i) - \mu|] \\ \label{eq:jgap-upper-bound} 
  &\leq \tilde{L}
\end{align}
In Eq.~(\ref{eq:jgap-lipschitz-softplus}), we used the 1-Lipschitz continuity of the softplus function; $|s(f_u(i')-f_u(i)) - s(\mu)| \leq |f_u(i')-f_u(i)-\mu|$.
In Eq.~(\ref{eq:jgap-upper-bound}), we use $-\tilde{L} \leq f_u(i') - f_u(i) \leq \tilde{L}$.
This result indicates that, by introducing norm clipping, we can bound the gap between our approach and pairwise counterpart with the same negative sampler $p_f(i|u)$.
As a summary, we can expect that our approach performs similarly to the pairwise counterpart with the top-weighted adaptive sampler.

\section{Experiments\label{expr}}
\subsection{Experimental Settings}
\subsubsection{Datasets}
To conduct a fair comparison,
we closely follow the settings of the experiments in LightGCN~\cite{he2020lightgcn} and DREGN-CF~\cite{togashi2021density}.
We utilise the same datasets and train/test splits that are available for these projects on GitHub\footnote{\url{https://github.com/kuandeng/LightGCN}\label{repo-lightgcn}};
(a) Gowalla~\cite{liang2016modeling}; (b) Yelp2018~\footnote{\url{https://www.yelp.com/dataset/challenge}}; and (c) Amazon-Book~\cite{he2016ups}.
The statistics are shown in Table~\ref{table:statistics}.

\subsubsection{Recommenders and Model Settings}
We list state-of-the-art methods based on neural networks as the baselines.
\begin{itemize}
\item\textbf{Neural Graph Collaborative Filtering (NGCF)}~\cite{NGCF} This is a GNN-based method that inherits the design of a graph convolutional network~\cite{GCN}. 
\item\textbf{Mult-VAE}~\cite{liang2018variational} This is a VAE-based collaborative filtering method. This method can be considered as a DE approach.
\item\textbf{Efficient Neural Matrix Factorisation (ENMF)}~\cite{chen2019efficient} This method is based on the non-sampling CPE approach based on Neural Collaborative Filtering (NCF)~\cite{NCF}.
\item\textbf{LightGCN}~\cite{he2020lightgcn} This is a state-of-the-art GNN-based method. It adopts the BPR loss~\cite{rendle2009bpr} with uniform negative sampling.
\item\textbf{DREGN-CF}~\cite{togashi2021density} This is a state-of-the-art LightGCN-based method with ranking DRE-based risk. It uses the same sampling strategy we discussed in Section~\ref{section:risk-approx} and a ranker-based weighting strategy.
\end{itemize}
We utilise the implementation of each method released by the authors in GitHub for these four methods: NGCF~\footnote{\url{https://github.com/xiangwang1223/neural_graph_collaborative_filtering.git}}, Mult-VAE~\footnote{\url{https://github.com/dawenl/vae_cf}}, ENMF~\footnote{\url{https://github.com/chenchongthu/ENMF}}, and LightGCN~\footnote{\url{https://github.com/kuandeng/LightGCN.git}}\footnote{\url{https://github.com/gusye1234/LightGCN-PyTorch}}.
For DREGN-CF, we implemented the ranking uLSIF risk without importance sampling and with non-negative risk correction, which empirically showed better ranking performance in their experiments. 
For NGCF, ENMF, LightGCN, and DREGN-CF, we utilise the best parameters reported by the authors for each dataset, when available;
as we implemented an ablated variant of DREGN-CF, we set $|U_B|$ to $2500$ and tuned hyper-parameters including the learning rate, coefficient of $L_2$ regularisation, and upper bound of density ratios.
For Mult-VAE, we utilised the model architecture suggested in the paper.

Our proposed method is based on LightGCN but trained with our proposed risk and norm clipping technique; we call our method the \emph{parametric density estimation LightGCN (\textbf{PDE-LGCN})}.
We also examine a method discussed in Section~\ref{section:disjoint-support}, namely, \emph{Wasserstein distance LightGCN (\textbf{WD-LGCN})}. 
For the methods based on LightGCN (i.e. LightGCN, DREGN-CF, PDE-LGCN, and WD-LGCN), we used the same network architecture, and we set the number of light-weight graph convolution layers and the dimensionality of embeddings to 3 and 64, respectively.
For our PDE-LGCN and WD-LGCN, we have two hyper-parameters, namely, the weight for the $L_2$ regularisation $\lambda$ and the upper bound of $L_2$ norms $\bar{n}$.
We tuned $\lambda$ and $\bar{n}$ based on validation within the ranges of $\{0.01, 0.05, 0.1\}$ and $\{1.0, 2.0, \dots, 9.0\}$.
We set $|U_B|$ to $2500$ for Gowalla, Yelp2018, and Amazon-Book.
For all datasets, the learning rate was tuned in $\{0.01, 0.02, \cdots, 0.1\}$.

\begin{table}[t]
\caption{Statistics of the four datasets: Gowalla, Yelp2018, Amazon-Book.}
  \centering
\begin{tabular}{lrrrr}
  \hline
  \multicolumn{1}{l}{Dataset} & \multicolumn{1}{r}{User \#} & \multicolumn{1}{r}{Item \#} & \multicolumn{1}{r}{Interaction \#} & \multicolumn{1}{r}{Density} \\ \hline
  Gowalla & $29,858$ & $40,981$ & $1,027,370$ & $0.00084$ \\ 
  Yelp2018 & $31,668$ & $38,048$ & $1,561,406$ & $0.00130$ \\ 
  Amazon-Book & $52,643$ & $91,599$ & $2,984,108$ & $0.00062$ \\ \hline
    \end{tabular}
  \label{table:statistics}
\end{table}

\subsection{Comparison with Baselines\label{section:expr-effectiveness}}
We first evaluate our proposed method, PDE-LGCN, and baselines in terms of overall ranking effectiveness. Following previous works~\cite{NGCF,he2020lightgcn,togashi2021density},
we utilise Recall@20 and nDCG@20~\cite{jarvelin2002cumulated} as the evaluation measures throughout this study.
Table~\ref{table:overall_effectiveness} lists the R@20 and nDCG@20 of each method in the three datasets.
We report the statistical significance between our proposed methods and DREGN-CF (the best-performing baseline) by a paired Tukey HSD test with 95\% CI and Hedge's $g$ as the effect size~\cite{sakai2018laboratory} in terms of nDCG@20.

With all the datasets, our PDE-LGCN and WD-LGCN outperform all the baselines in terms of both R@20 and nDCG@20.
The gain from DREGN-CF is statistically significant in Yelp2018 and Amazon-Book despite using the same model architecture.
For PDE-LGCN, the p-value and effect size are $p<1e\mathrm{-}3$, $g=0.0415$ in Yelp2018 and $p<1e\mathrm{-}3$, $g=0.0520$ in Amazon-Book.
For WD-LGCN, $p<1e\mathrm{-}3$, $g=0.0689$ in Yelp2018 and $p<1e\mathrm{-}3$, $g=0.1194$ in Amazon-Book.
On the other hand, the difference is not significant in Gowalla;
$p=0.244$, $g=0.0301$ for PDE-LGCN and $p=0.528$, $g=0.0197$ for WD-LGCN.
The gap between DREGN-CF and our proposed methods is substantial particularly in the Amazon-Book dataset, which has the largest number of items in the datasets (see also Table~\ref{table:statistics}).
Because the quality of the negative sampling/weighting strategy has a large impact on the ranking performance for such a dataset,
this result suggests that the derived negative sampler/weighting of our approach is more effective than the weighting strategy of DREGN-CF, which is heuristically designed.

By comparing PDE-LGCN and WD-LGCN,
WD-LGCN outperforms PDE-LGCN with statistical significance in Amazon-Book ($p<1e\mathrm{-}3$, $g=0.0954$),
whereas the difference is not significant in Gowalla and Yelp2018;
$p=0.841$, $g=0.0121$ for Gowalla and $p=0.209$, $g=0.0341$ for Yelp2018.
The performance gain of WD-LGCN in Amazon-Book is remarkable.
This suggests that WD-LGCN is effective for a sparse dataset;
it is probably because, by estimating densities on unobserved items (i.e. $\mathcal{I}\setminus\mathcal{I}_u^{+}$),
WD-LGCN can avoid the concentration of the probability mass (i.e. $p_f(i|u)$) assigned to a handful of observed items.
Since a large-scale dataset inevitably experiences the sparsity of observed interactions,
the WD-based approach may be beneficial for large-scale settings.

\begin{table}[t]
  \caption{Comparison of methods in terms of ranking effectiveness. The best-performing baseline is DREGN-CF.}
  \label{table:overall_effectiveness}
  \resizebox{0.48\textwidth}{!}{
    \begin{tabular}{l|cc|cc|cc}
      \hline
      \multirow{2}{*}{Method} & \multicolumn{2}{c|}{Gowalla} & \multicolumn{2}{c|}{Yelp2018} & \multicolumn{2}{c}{Amazon-Book} \\ \cline{2-7}
                 & R@20   & nDCG@20& R@20   & nDCG@20& R@20   & nDCG@20 \\ \hline
      NGCF       & 0.1567 & 0.1325 & 0.0575 & 0.0474 & 0.0342 & 0.0262 \\ 
      Mult-VAE   & 0.1644 & 0.1340 & 0.0589 & 0.0458 & 0.0413 & 0.0306 \\ 
      ENMF       & 0.1512 & 0.1306 & 0.0628 & 0.0515 & 0.0344 & 0.0272 \\ 
      LightGCN   & 0.1828 & 0.1551 & 0.0651 & 0.0532 & 0.0421 & 0.0324 \\
      DREGN-CF   & 0.1828 & 0.1551 & 0.0685 & 0.0564 & 0.0506 & 0.0395 \\ \hline
      PDE-LGCN   & 0.1871    & 0.1575    & 0.0706    & 0.0582    & 0.0542    & 0.0422 \\
      WD-LGCN   & 0.1859    & 0.1567    & 0.0719    & 0.0594    & 0.0580    & 0.0466 \\ \hline
    \end{tabular}}
\end{table}

\begin{table*}[t]
  \caption{Comparison of overall ranking effectiveness among ablated methods.}
  \label{table:ablation}
  \resizebox{0.98\textwidth}{!}{
    \begin{tabular}{lccc|cccccc}
\hline
\multirow{2}{*}{Method}   & \multirow{2}{*}{\begin{tabular}[c]{@{}c@{}}Model\end{tabular}} & \multirow{2}{*}{\begin{tabular}[c]{@{}l@{}}Risk\end{tabular}} & \multirow{2}{*}{norm clipping} & \multicolumn{2}{c}{Gowalla} & \multicolumn{2}{c}{Yelp2018} & \multicolumn{2}{c}{Amazon-Book} \\ \cline{5-10} 
                          &                      &                                                                         &                     & R@20        & nDCG@20       & R@20        & nDCG@20        & R@20          & nDCG@20         \\ \hline
      (a) PDE-MF w/o NC	  & MF	 & PDE & & 0.1377 & 0.1097 & 0.0514 & 0.0431 & 0.0443 & 0.0349 \\	
      (b) PDE-MF	  & MF	 & PDE & \checkmark  & 0.1512 & 0.1224 & 0.0647 & 0.0525 & 0.0504 & 0.0398 \\
      (c) WD-MF	  & MF	  & WD & \checkmark  & 0.1465 & 0.1161 & 0.0607 & 0.0513 & 0.0476 & 0.0401 \\ 
      (d) LightGCN-ANS	  & LGCN & Pairwise &		  & 0.1868 & 0.1568 & 0.0677 & 0.0561 &	0.0504 & 0.0389 \\  
      (e) LightGCN-ANS-NC & LGCN & Pairwise & \checkmark  & 0.1895 & 0.1601 & 0.0698 & 0.0578 & 0.0499 & 0.0386 \\ 
      (f) PDE-LGCN w/o NC & LGCN & PDE &		  & 0.1763 & 0.1493 & 0.0693 & 0.0569 & 0.0526 & 0.0414 \\ \hline
      (g) PDE-LGCN	  & LGCN & PDE & \checkmark  & 0.1871 & 0.1575 & 0.0706 & 0.0582 & 0.0542 & 0.0421 \\ 
      (h) WD-LGCN	  & LGCN & WD & \checkmark  & 0.1859 & 0.1567 & 0.0719 & 0.0594 & 0.0580 & 0.0466 \\ \hline
    \end{tabular}
  }
\end{table*}

\begin{figure*}[t]
    \centering
    \includegraphics[clip,width=\linewidth]{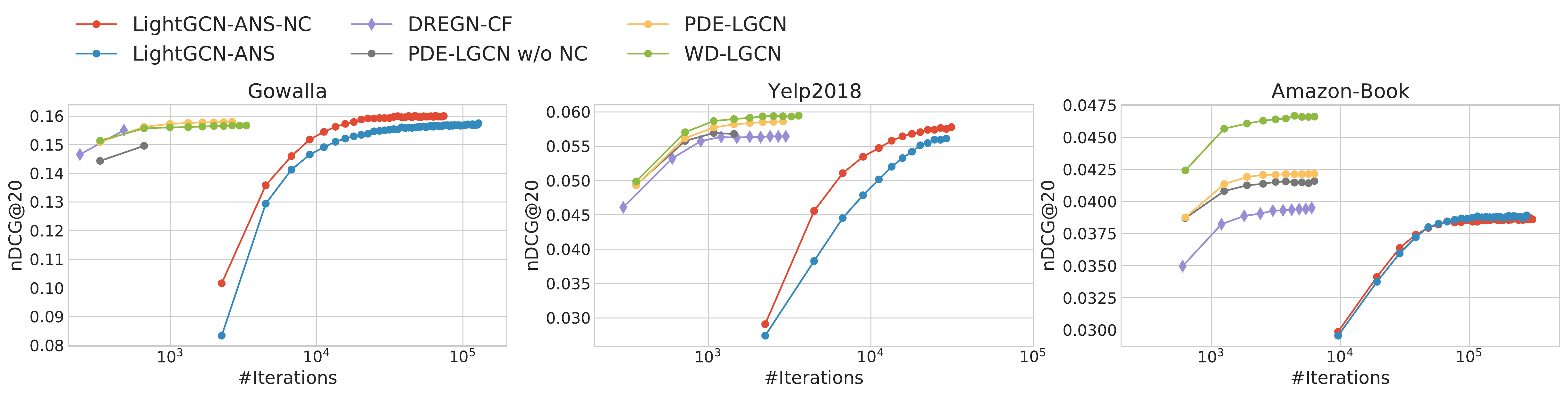}
    \caption{Comparison between LightGCN, DREGN-CF, PDE-LGCN, and WD-LGCN in terms of training efficiency.}
    \label{fig:training-efficiency}
\end{figure*}

\subsection{Ablation Study\label{section:expr-ablation}}
In this section, we compare PDE-LGCN and WD-LGCN with ablated variants of them.
In particular, based on the discussion in Section~\ref{section:rel-pairwise},
we hypothesised that our approach performs similarly to the pairwise approach with an adaptive negative sampler. 
To examine the hypothesis, we implemented an adaptive negative sampling (ANS) technique for the variants based on the pairwise approach.
We first sample negative candidates based on the same sampling strategy for $|I_B|$ discussed in Section~\ref{section:risk-approx},
and then, for each user, re-sample five unobserved items according to the softmax probabilities in a mini-batch (i.e. $\exp(f_u(i))/\sum_{i \in \mathcal{I}_B}\exp(f_u(i))$) based on the predicted scores. 

We list six ablated variants:
(a) MF with the PDE risk without norm clipping (PDE-MF w/o NC);
(b) MF with the PDE risk with norm clipping (PDE-MF);
(c) MF with the WD risk with norm clipping (WD-MF);
(d) LightGCN with the pairwise risk and ANS without norm clipping (LightGCN-ANS);
(e) LightGCN with the pairwise risk and ANS with norm clipping (LightGCN-ANS-NC); and
(f) LightGCN with the PDE risk without norm clipping (PDE-LGCN w/o NC).
We omitted the WD-LGCN and WD-MF without norm clipping because it often fails the model training.

Table~\ref{table:ablation} lists the R@20 and nDCG@20 of each method.
Our proposed methods are labelled (g) PDE-LGCN and (h) WD-LGCN.
We report statistical significance and effect sizes as in Section~\ref{section:expr-effectiveness}.
For the variants with the pairwise risk,
we increase $|\mathcal{U}_B|$ to 4096 for speed and ranking performance.

\subsubsection{Effectiveness of the PDE and WD risks\label{section:effectiveness-pde}}
From the comparison between (e) LightGCN-ANS-NC and each of (g) PDE-LGCN and (h) WD-LGCN,
PDE-LGCN and WD-LGCN show comparable or substantially better performance than LightGCN-ANS-NC.
The improvement is particularly clear in the Amazon-Book dataset ($p<1e\mathrm{-}3$, $g=0.0781$ for PDE-LGCN and $p<1e\mathrm{-}3$, $g=0.1306$ for WD-LGCN).
Our methods also outperform LightGCN-ANS-NC in Yelp2018 although the differences are not significant ($p=0.900$, $g=0.0165$ for PDE-LGCN and $p=0.0576$, $g=0.0388$ for WD-LGCN).
On the other hand, our methods underperform LightGCN-ANS-NC by a relatively small margin in the Gowalla dataset.
However, the differences are not statistically significant ($p=0.597$, $g=0.0359$ for PDE-LGCN and $p=0.259$, $g=0.0447$ for WD-LGCN).

As PDE-LGCN, WD-LGCN and LightGCN-ANS-NC have the similar negative sampling distributions and mini-batch sampling strategy,
this result suggests that our approach is particularly beneficial for use with large datasets. 
For a small dataset, our approach might underperform the pairwise loss with an appropriate sampling distribution and norm clipping technique.

Remarkably, in the Amazon-Book dataset,
(b) PDE-MF and (c) WD-MF outperform most of the conventional neural network-based methods despite it being based on a simple MF (see Table~\ref{table:overall_effectiveness}).
In particular, PDE-MF and WD-MF clearly outperform Mult-VAE, which is based on a DE approach.
This is a surprising result because
(1) PDE-MF and WD-MF are based on simple MF; and
(2) for scalability, our risk estimator defined in Eq.~(\ref{eq:risk-estimator}) and Eq.~(\ref{eq:risk-estimator-ds}) introduces the approximation of $\exp A(f_u)$ which is exactly computed in Mult-VAE.
This result further emphasises the effectiveness of our approach.

By comparing (b) PDE-MF and (c) WD-MF,
PDE-MF outperforms WD-MF in Gowalla and Yelp2018;
$p<1e\mathrm{-}3$, $g=0.0597$ for Gowalla and $p=0.626$, $g=0.0174$ for Yelp2018. 
PDE-MF underperforms WD-MF in Amazon-Book; $p=0.9$, $g=0.0050$.
Considering that the difference is not significant in Yelp2018 and Amazon-Book,
we recommend the use of the PDE risk for MF-based methods.

\subsubsection{Efficiency of the PDE and WD risks\label{section:efficiency-pde}}
Figure~\ref{fig:training-efficiency} shows the comparison between PDE-LGCN, WD-LGCN, DREGN-CF and LightGCN-ANS with and without norm clipping in each dataset.
The x- and y-axes in each figure indicate the number of training iterations in a logarithmic scale and the testing nDCG@20, respectively.
We run the experiments with one NVIDIA Tesla P100 GPU, 16 Intel(R) Xeon(R) CPU @ 2.20GHz, and 60GB of RAM.
For all the datasets, PDE-LGCN and WD-LGCN showed faster convergence than the LightGCN-ANS variants while achieving competitive ranking performance.
PDE-LGCN converges after about 3,000 iterations and this takes 225, 261 and 522 seconds for Gowalla, Yelp2018 and Amazon-Book, respectively.
WD-LGCN converges after about 4,000 iterations and this takes 301, 350 and 696 seconds for Gowalla, Yelp2018 and Amazon-Book, respectively.
DREGN-CF converges with approximately 480, 2,000, and 2,000 iterations in 110, 462, and 982 seconds for Gowalla, Yelp2018, and Amazon-Book, respectively.
Particularly in the Amazon-Book dataset, WD-LGCN and PDE-LGCN outperform DREGN-CF in terms of convergence time.
This is because DREGN-CF takes much longer inference time than our risk estimators;
the ranking uLSIF risk of DREGN-CF requires six terms involving a self-normalised weighted average, whereas our risk estimators have two average terms for each user.
LightGCN-ANS-NC converges with approximately 50,000, 29,120 and 150,000 iterations in 4,843, 3,214, and 33,229 seconds for Gowalla,Yelp2018 and Amazon-Book, respectively.
Our PDE-LGCN and WD-LGCN converge substantially faster than LightGCN with the pairwise loss and an adaptive negative sampler,
and demonstrate comparative training efficiency when compared with DREGN-CF.

As a summary, with considering also the results in Section~\ref{section:effectiveness-pde},
our approach achieves comparative performance and faster convergence than the pairwise counterpart.
Furthermore, it clearly outperforms the DRE-based approach in terms of effectiveness with comparative or faster model training.

\subsubsection{Effectiveness of the norm clipping Technique\label{section:expr-nc}}
Comparing (f) and (g), PDE-LGCN consistently deteriorates without the norm clipping technique, whereas the difference is not significant in Yelp2018 and Amazon-Book. 
$p<1e\mathrm{-}3$, $g=0.106$ for Gowalla, $p=0.636$, $g=0.0335$ for Yelp2018, and $p=0.874$, $g=0.0204$ for Amazon-Book.
On the other hand, by comparing (a) PDE-MF without NC and (b) PDE-MF,
PDE-MF significantly outperforms that without norm clipping in all datasets;
$p<1e\mathrm{-}3$, $g=0.116$ for Gowalla, $p<1e\mathrm{-}3$, $g=0.0639$ for Yelp2018, and $p<1e\mathrm{-}3$, $g=0.0747$ for Amazon-Book.
From these results, the norm clipping technique is beneficial for our proposed approach, whereas the improvement is relatively small with LightGCN.

By introducing LightGCN, our approach further improves the ranking performance from PDE-MF for all datasets;
the results of the comparison between (b) PDE-MF and (g) PDE-LGCN are  
$p<1e\mathrm{-}3$, $g=0.283$ for Gowalla, $p<1e\mathrm{-}3$, $g=0.0639$ for Yelp2018, and $p<1e\mathrm{-}3$, $g=0.0366$ for Amazon-Book.
These results may imply that LightGCN leads to the smoothness of predicted scores by sharing the latent vectors over the user-item graph
and plays a role of inducing the Lipschitz continuity of $f_u$.  
However, the effect of graph-based locality modelling may also be a cause of the significant improvement in Gowalla, which contains the user-venue check-in data collected on a location-based social network~\cite{liang2016modeling}.

For WD-LGCN and WD-MF, we found that the model training is quite challenging without norm clipping;
this may be because the objective based on the WD in Eq.~(\ref{eq:wass-obj}) is not continuous and differentiable for the distributions with disjoint supports without the Lipschitz continuity of $f_u$~\cite{arjovsky2017wasserstein}.
From these results, our norm clipping technique is vital for the proposed WD approach.
Furthermore, due to its applicability, it allows us to take an advantage of various models such as GNNs.

To further investigate the effect of $\bar{n}$,
we demonstrate the hyper-parameter sensitivity of PDE-LGCN and WD-LGCN by varying $\bar{n}$ from $1$ to $9$ while keeping other parameters fixed.
Figure~\ref{fig:sensitivity_n} shows the effect of $\bar{n}$ on the validation performance of PDE-LGCN and WD-LGCN for each dataset.
The x- and y-axes in each figure indicate $\bar{n}$ and validation performance, respectively.
The figures in the top and bottom demonstrate the results of PDE-LGCN and WD-LGCN, respectively.
The red broken line in each figure indicates the performance with $\bar{n}=\infty$ (i.e. PDE-LGCN w/o NC); we omit it for WD-LGCN due to unstable model training.
For all the datasets, a small value ($\bar{n}=1,2,3$) leads to significant deterioration in terms of nDCG@20 for both PDE-LGCN and WD-LGCN.
We can observe improvement with an appropriate value of $\bar{n}$ in each dataset,
whereas the performance deteriorates with a larger $\bar{n}$ and approaches the performance seen with $\bar{n}=\infty$ for PDE-LGCN.
These results also indicate the effectiveness of the norm clipping technique, and therefore,
we recommend that $\bar{n}$ be tuned when using our approach;
the appropriate value of $\bar{n}$ tends to be large when the number of items in a dataset is large.

\begin{figure}[h]
  \begin{subfigure}{0.48\textwidth}
  \centering
  \includegraphics[clip,width=\linewidth]{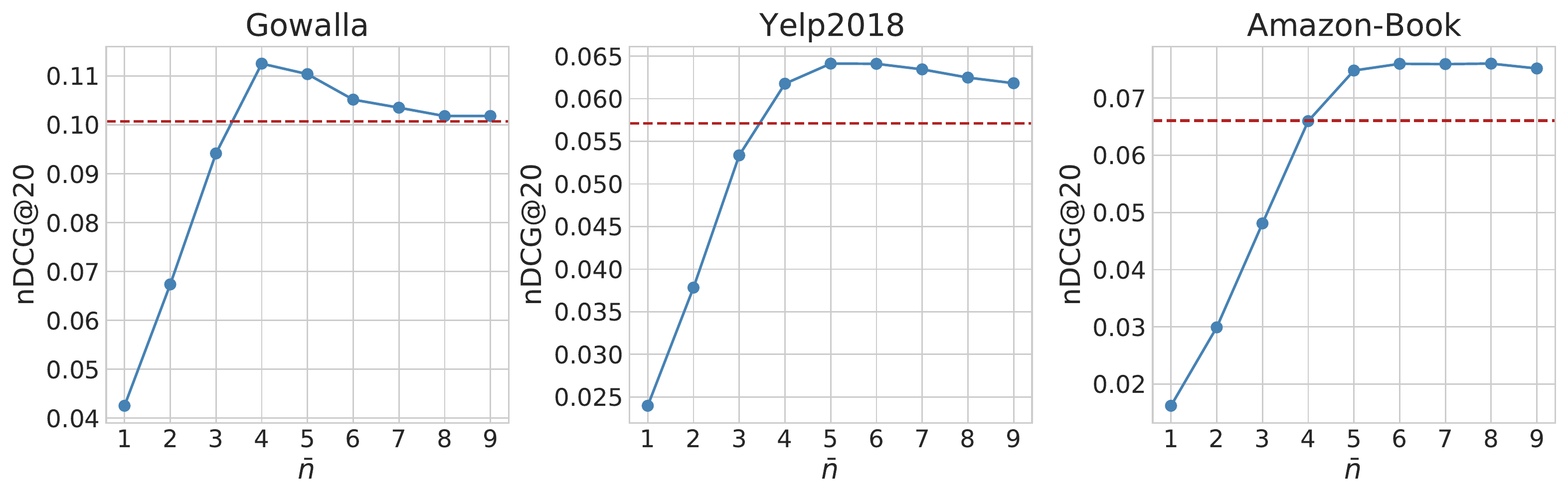}    
\end{subfigure}
\begin{subfigure}{0.48\textwidth}
  \centering
  \includegraphics[clip,width=\linewidth]{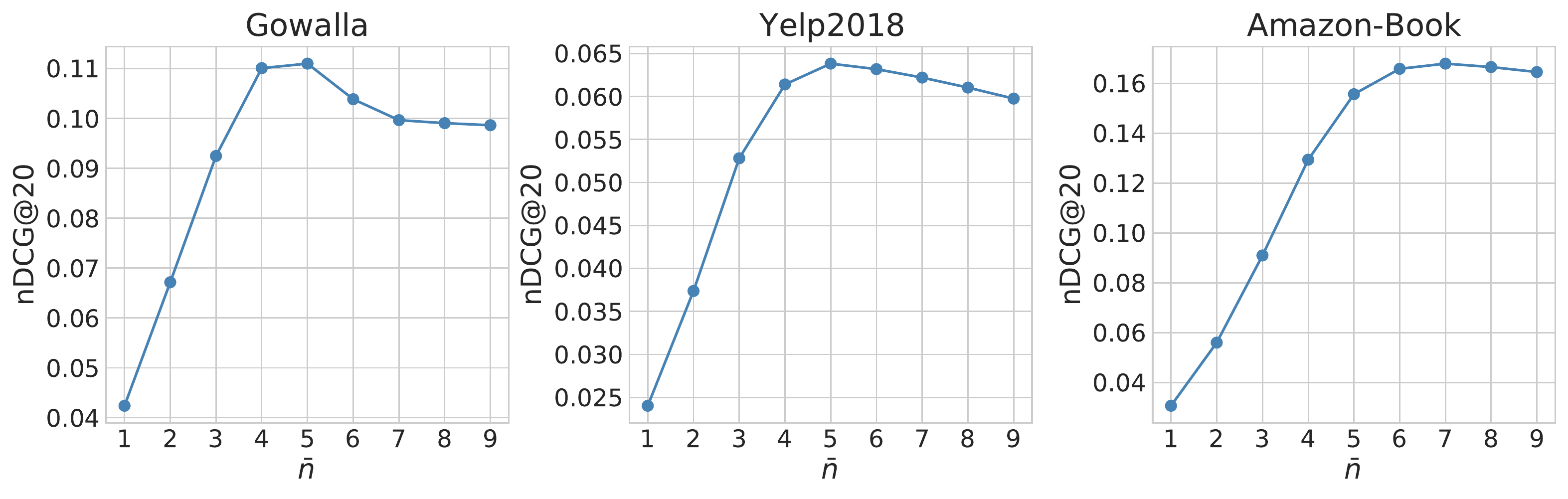}
\end{subfigure}
\caption{Effect of $\bar{n}$ on nDCG@20 of PDE-LGCN and WD-LGCN on the validation split.
The figures in the top and bottom demonstrate the results of PDE-LGCN and WD-LGCN, respectively.
The red broken lines indicate the performance without the norm clipping ($\bar{n}=\infty$); we omit it for WD-LGCN due to unstable model training.
}
\label{fig:sensitivity_n}
\end{figure}

\section{Conclusion}
We proposed a scalable learning-to-rank approach for implicit feedback.
We formulated the problem as the estimation of density functions within an exponential family and developed a loss function based on the penalised MLE of generative distributions of users' positive items.
Our approach includes the following features:
(a) at every training step, the ranker can act as the optimal sampler for itself, thereby omitting extra sampler models; and
(b) the risk estimator efficiently approximates the expectations over ranker-dependent distributions by exploiting the explicit parametric density function.
Furthermore, we proposed a simple regularisation technique for MF and GNNs, namely, norm clipping, which smooths the model predictions and stabilises training.
By deriving the proposed risk from a WD minimisation perspective,
we addressed the essential problem in the ranking learning/evaluation;
we need to know the ranking of only unobserved items.
Using empirical analyses on three real-world datasets, we demonstrated the ranking effectiveness and training efficiency of PDE-LGCN and WD-LGCN.
Our methods achieve comparable or substantially better ranking performance than the conventional methods, while drastically reducing the training time.

As future work, we will further explore the regularisation technique and model architectures based on our proposed risk;
the discussion in Section~\ref{section:rel-wgan} and Section~\ref{section:expr-nc} suggest that the smoothness of model prediction is essential for our approach.
Based on the discussion in Section~\ref{section:rel-wgan},
we will theoretically investigate the duality gap between IRWGAN and our approach due to the non-convex function spaces when using complex models~\cite{bose2020adversarial}.
Since we adopted the static mini-batch sampling,
it would be interesting to examine advanced algorithms~\cite{ding2020simplify,xiong2020approximate} for our approach.

\clearpage
\bibliographystyle{ACM-Reference-Format}
\bibliography{head.blb}

\end{document}